\documentclass{article}
\usepackage{iclr2027_conference,times}

\usepackage{amsmath,amsfonts,bm}

\def\eqref#1{equation~\ref{#1}}

\def\1{\bm{1}}

\DeclareMathAlphabet{\mathsfit}{\encodingdefault}{\sfdefault}{m}{sl}
\SetMathAlphabet{\mathsfit}{bold}{\encodingdefault}{\sfdefault}{bx}{n}

\usepackage[utf8]{inputenc}
\usepackage[T1]{fontenc}
\usepackage{hyperref}
\usepackage{url}
\usepackage{booktabs}
\usepackage{amsfonts}
\usepackage{amssymb}
\usepackage{amsmath}
\usepackage{amsthm}
\usepackage{nicefrac}
\usepackage{microtype}
\usepackage{xcolor}
\usepackage{graphicx}
\usepackage{caption}
\usepackage{subcaption}

\newtheorem{proposition}{Proposition}
\newtheorem{lemma}{Lemma}
\newtheorem{corollary}{Corollary}
\theoremstyle{remark}
\newtheorem{remark}{Remark}
\theoremstyle{plain}

\title{CyFM: Cylindrical Optimal Transport for\\Few-Step Complex-Valued Flow Matching}

\author{
  Marcel Musia{\l}ek, Iga Wolanin, Damian Ryczko, Anna Grelewska \\
  Wroc{\l}aw University of Science and Technology, Poland \\
  \texttt{\{marcel.musialek, 281203, 281188, 279694\}@student.pwr.edu.pl} \\
  \And
  Oleksii Furman \\
  Wroc{\l}aw University of Science and Technology, Poland \\
  \texttt{oleksii.furman@pwr.edu.pl}
}

\newenvironment{ack}{\subsubsection*{Acknowledgments}}{}

\iclrfinalcopy 
\begin{document}

\maketitle

\begin{abstract}
Complex-valued MRI and audio signals are commonly modeled as two Cartesian channels, although their phase is circular. In Cartesian flow matching, straight paths between nearly opposite phases can pass close to the origin and induce arbitrarily large angular velocities. We formulate \textbf{Cylindrical Flow Matching (CyFM)} on amplitude and phase coordinates $[0,\infty)\times S^1$ with the decoupled metric $dA^2+d\theta^2$. Shortest-arc phase interpolation bounds the angular training target by $\pi$, while exact minibatch Optimal Transport (OT) pairs noise and data jointly over whole fields. For equal-amplitude Cartesian endpoints with uniform phase differences, we derive an index-$1$ Pareto tail in peak angular velocity. Across exact bridges on synthetic copula fields, LibriSpeech spectrograms, and fastMRI knee wavefields, $42.6\%$--$49.4\%$ of target signal energy lies on independently coupled paths whose peak angular velocity exceeds $\pi$; no cylindrical target exceeds this bound. Under matched training budgets, CyFM with joint OT achieves lower pooled sliced $W_2$ than the best Cartesian baseline at every tested solver step $k\leq8$ on synthetic fields at three resolutions and on speech spectrograms. On $64\times64$ knee MRI, its pooled error is $1.8\times$ lower at one step; amplitude texture favors CyFM from four steps, and phase coherence from eight steps. Joint OT reduces synthetic few-step error by $3\%$--$60\%$, whereas factorizing the coupling disrupts joint dependencies. At full $320\times320$ MRI resolution, CyFM preserves a $2.1\times$ single-step advantage, demonstrating that cylindrical geometry and joint transport provide an effective inductive bias for rapid, few-step generation of complex-valued physical fields.
\end{abstract}

\section{Introduction}
\label{sec:intro}

Complex-valued physical signals $z=x+iy=Ae^{i\theta}\in\mathbb{C}$, including magnetic resonance imaging (MRI) wavefields and complex speech spectrograms, encode information in both amplitude and phase. Generative models commonly represent each coefficient as two Cartesian channels, although phase is circular and becomes undefined at the origin. Flow Matching \citep{lipman2023flow, albergo2023building, liu2023flow} provides a way to study the effect of this choice: it learns continuous-time generative flows without simulating trajectories during training and extends to Riemannian manifolds \citep{chen2024flow, mathieu2020riemannian}.

For a straight Cartesian probability path $z_t=(1-t)z_0+t z_1$, endpoints with similar amplitudes and nearly opposite phases force the path close to the origin, attenuating its amplitude. Its induced angular velocity can then grow without bound even though the Cartesian target velocity $z_1-z_0$ remains bounded for bounded endpoints. Each solver step requires a network evaluation, so accurate generation with few steps reduces sampling cost and provides a practical test of this path geometry. Figure~\ref{fig:teaser_flow} illustrates the contrast.

\begin{figure}[t]
    \centering
    \begin{subfigure}[b]{0.48\textwidth}
        \centering
        \includegraphics[height=2.9cm, keepaspectratio]{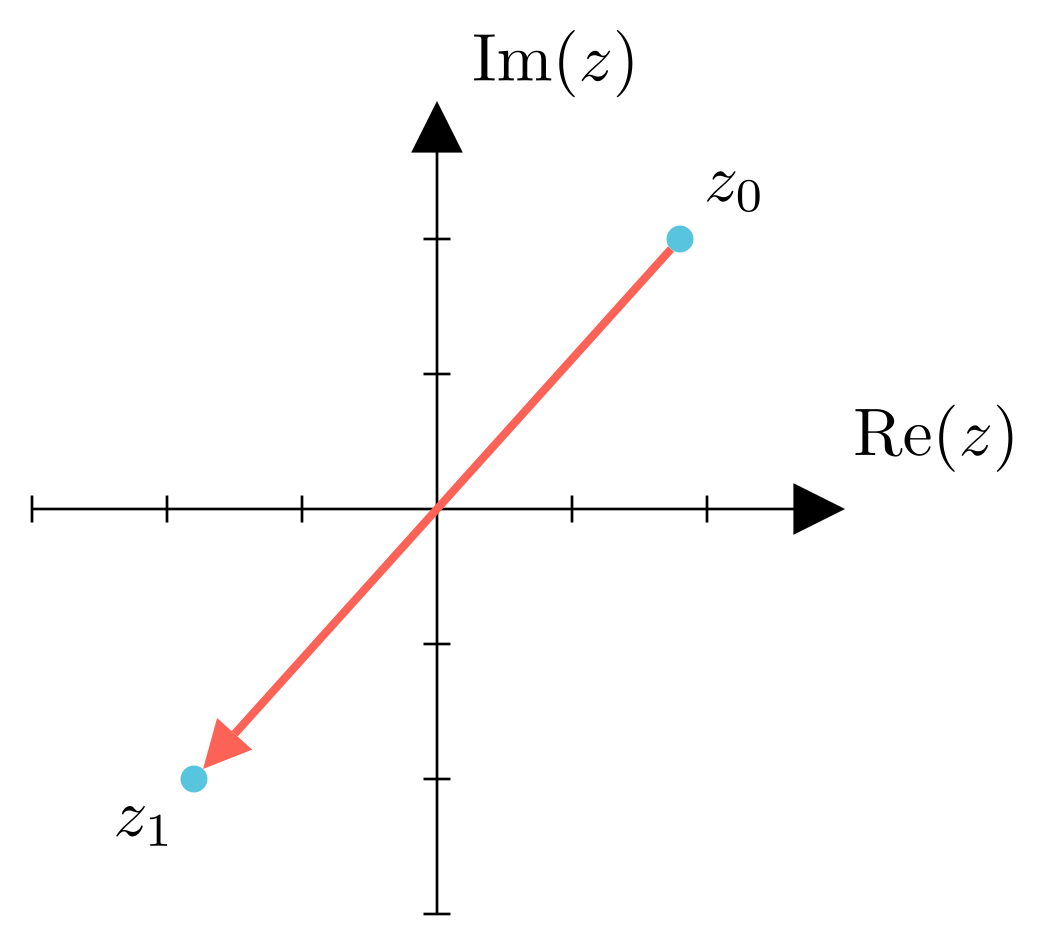}
        \caption{Standard Euclidean Flow ($\mathbb{R}^2$)}
        \label{fig:flow_euclidean}
    \end{subfigure}
    \hfill
    \begin{subfigure}[b]{0.48\textwidth}
        \centering
        \includegraphics[height=2.9cm, keepaspectratio]{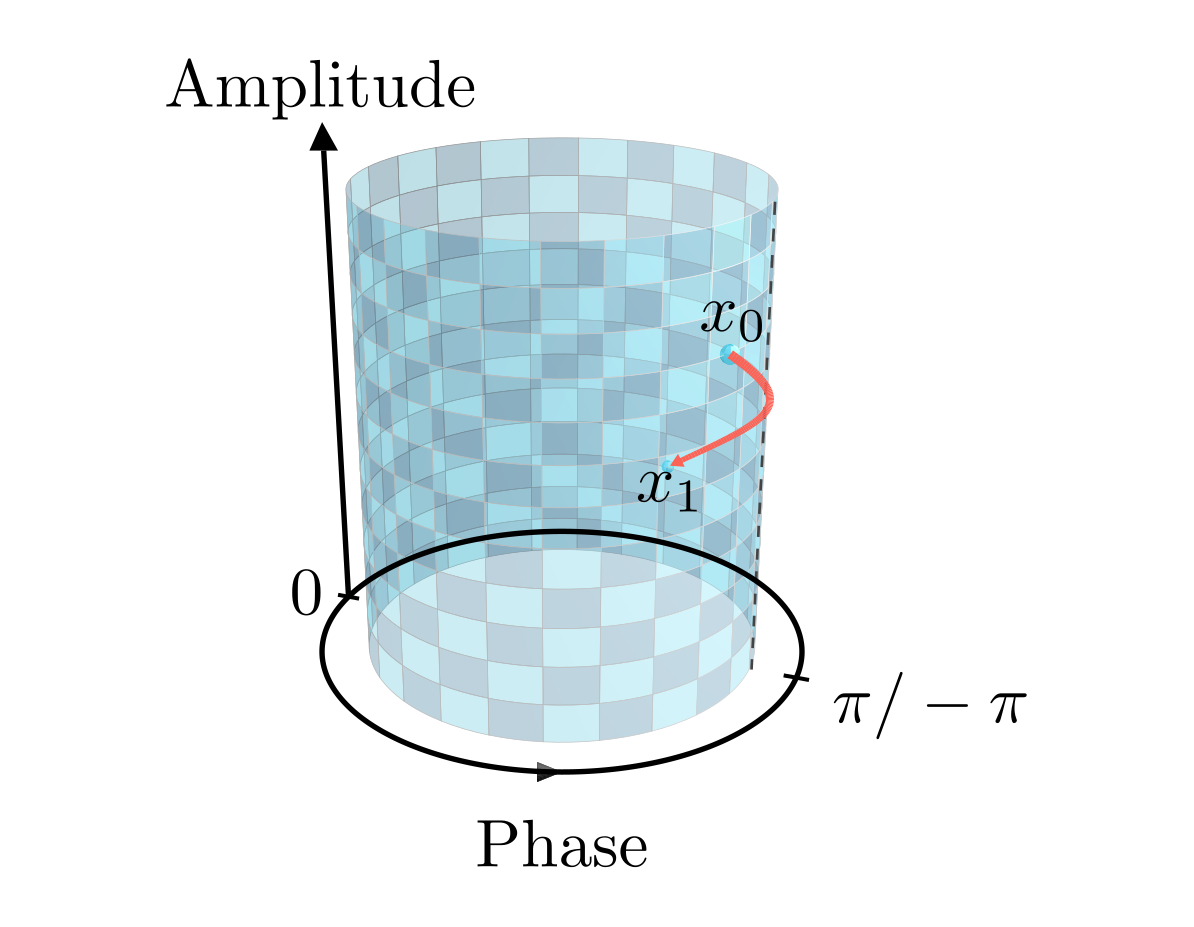}
        \caption{Cylindrical Flow Matching (Ours)}
        \label{fig:flow_cylindrical}
    \end{subfigure}
    \caption{Geometric mismatch in complex-valued generative flows. \textbf{(a)} In the Euclidean complex plane $(\mathrm{Re}(z), \mathrm{Im}(z))$, linear probability trajectories pass close to the origin, inducing severe midpoint amplitude attenuation ($|z_{1/2}| \ll \frac{1}{2}(|z_0| + |z_1|)$). \textbf{(b)} Cylindrical Flow Matching works in amplitude--phase coordinates on the cylindrical closure $[0, \infty) \times S^1$ under the decoupled product metric, interpolating amplitude directly and transporting phase along the shortest circular geodesic.}
    \label{fig:teaser_flow}
    \vspace{-2mm}
\end{figure}

Decoupled cylindrical or polar metrics suggest an alternative to Cartesian paths \citep{chen2026dpfm, wang2026disrfm}. For high-dimensional complex fields, however, the pairing of noise and data samples also matters: it determines the phase differences and angular velocities along each bridge. Minibatch Optimal Transport (OT) can favor shorter pairings, but its transport-cost advantage may shrink with field dimension. Assigning coordinates or spatial patches separately may also break dependencies within each field. This leads to two questions. First, does a bounded angular target improve generation when only a few solver steps are available? Second, how does the benefit of whole-field OT change as field dimension grows, and what spatial structure is lost when the assignment is factorized?

We address these questions with \textbf{Cylindrical Flow Matching (CyFM)} on $\mathcal{M}=[0,\infty)\times S^1$ under the decoupled metric $dA^2+d\theta^2$. CyFM interpolates phase along the shortest circular arc and uses exact minibatch OT to pair complete noise and data fields. The product metric keeps phase cost independent of amplitude, including at the boundary $A=0$; this choice can help few-step generation but may also give undue weight to phase where a zero coefficient contains no measurement. We test both effects through exact bridge analysis and trained models on synthetic fields, speech short-time Fourier transform (STFT) spectrograms, and fastMRI knee wavefields at cropped and full resolution.

Our contributions are as follows:
\begin{itemize}

    \item We prove that CyFM target angular velocity is bounded by $\pi$ for any coupling, derive the peak angular velocity of Cartesian bridges, and obtain an index-$1$ Pareto tail for equal-amplitude endpoints with uniform phase differences. Exact bridges from all three data domains quantify how often Cartesian paths exceed the cylindrical bound.
    
    \item We pair complete fields with exact minibatch OT. On synthetic fields, joint coupling reduces few-step error at every tested resolution even as its transport-cost advantage narrows with dimension. Factorizing the assignment across amplitude, phase, or patches can change the joint target distribution. In our tests, patchwise matching disrupts correlations across patch boundaries.
    
    \item Under matched training budgets, CyFM achieves lower pooled sliced $W_2$ than the best Cartesian baseline at every tested $k\leq8$ on synthetic fields at three resolutions and on speech, with all five seeds separated. On $64\times64$ knee MRI, it leads on the pooled metric at $k=1$, on amplitude texture from $k=4$, and on phase coherence from $k=8$.

    \item At full $320\times320$ knee MRI resolution, CyFM retains a $2.1\times$ single-step pooled advantage but trails Cartesian flows on all measured spatial criteria. This result identifies a limit of the decoupled representation on fields with zero-valued, unmeasured background.
\end{itemize}

\section{Related Work}
\label{sec:related_work}

\textbf{Flow Matching and Riemannian Manifolds.} Continuous Normalizing Flows (CNFs) parametrise generative trajectories via neural ODEs \citep{chen2018neural}. Flow Matching (FM) \citep{lipman2023flow, albergo2023building, liu2023flow} introduced simulation-free regression onto closed-form target paths, with minibatch Optimal Transport (OT) further straightening trajectories \citep{pooladian2023multisample, tong2023improving}. Riemannian Flow Matching (RFM) \citep{chen2024flow} generalized this to Riemannian manifolds. Recently, decoupled polar product metrics ($\mathbb{R}^+ \times S^{d-1}$) have appeared in latent feature alignment, such as Direct Product Flow Matching (DP-FM) \citep{chen2026dpfm} for vision-language models and DisRFM \citep{wang2026disrfm} for graph domain adaptation. However, applying product manifolds directly to high-dimensional raw complex wavefields involves unique spatial coherence and OT scaling dynamics that remain unexplored.

\textbf{Generative Models for Complex-Valued Signals.} Physical wavefields in MRI and audio STFT inherently require modelling amplitude and phase \citep{bernstein2004handbook, trabelsi2018deep}. While complex convolutions have been explored in supervised settings \citep{cole2021analysis}, modern generative frameworks (including Score-Based Diffusion, \citealp{song2020score, chung2022score}) typically operate on amplitude alone or embed complex numbers as flat Cartesian channels $(\mathrm{Re}, \mathrm{Im}) \in \mathbb{R}^2$, inheriting planar geometry. Recent complex generative models either project phase into a learned Euclidean latent space \citep{schlimbach2026complex} or default to flat Cartesian interpolation \citep{rempe2025phasegen, choi2025universr}. As we demonstrate, Cartesian interpolation structurally induces origin singularities and extreme angular velocities.

\textbf{Optimal Transport Couplings and Factorisation Traps.} Minibatch OT pairings straighten probability paths in Euclidean space \citep{tong2023improving, pooladian2023multisample} and on structured manifolds \citep{bose2024foldflow, kapusniak2024metric, boite2026minibatch}. To alleviate the quadratic cost of batch OT, alternative couplings such as Quantile Coupling (QC-FM, \citealp{kim2026onesided}) pair samples via one-dimensional projections. A tempting shortcut in high dimensions is to factorise transport independently across coordinates (amplitude/phase) or spatial patches. We prove and demonstrate empirically that factorising coupling reassembles independent marginals from disjoint samples, destroying joint copulas and spatial seams (Section~\ref{sec:experiments}).

\section{Preliminaries and Geometric Foundations}
\label{sec:preliminaries}

The geometry of a probability path determines the velocity field that Flow Matching asks a model to learn. To make this dependence explicit, we first recall the Euclidean objective and its straight conditional paths. We then describe amplitude--phase coordinates and distinguish the inherited Euclidean polar metric from the product metric used by CyFM, including their behavior at zero amplitude. Finally, we quantify the amplitude attenuation caused by Cartesian interpolation, which motivates the bounded angular targets in Section~\ref{sec:method}.

\subsection{Flow Matching on Euclidean Spaces}
Let $p_0$ denote a tractable base distribution on $\mathbb{R}^d$, such as a standard Gaussian $\mathcal{N}(0, I)$ or a uniform prior, and let $p_1$ denote an empirical data distribution. Continuous Normalizing Flows model continuous-time transport between $p_0$ and $p_1$ through a time-dependent vector field $v_t: [0, 1] \times \mathbb{R}^d \to \mathbb{R}^d$, defining a flow $\psi_t: \mathbb{R}^d \to \mathbb{R}^d$ via the ordinary differential equation (ODE) $\frac{d}{dt} \psi_t(x) = v_t(\psi_t(x))$ with $\psi_0(x) = x$. Conditional Flow Matching (CFM) \citep{lipman2023flow, albergo2023building, liu2023flow}  bypasses the intractable marginal vector field by constructing target conditional vector fields $u_t(x_t|x_0, x_1) = x_1 - x_0$ along linear probability paths $x_t = (1 - t) x_0 + t x_1$. Then, a neural network $v_\phi(t, x)$ is trained through the simulation-free objective:
\begin{equation}
    \mathcal{L}_{\mathrm{CFM}}(\phi) = \mathbb{E}_{t \sim \mathcal{U}[0, 1],\, x_0 \sim p_0,\, x_1 \sim p_1} \left[ \| v_\phi(t, x_t) - u_t(x_t | x_0, x_1) \|^2 \right].
\end{equation}
Once trained, new samples are generated by drawing $x_0 \sim p_0$ and integrating $v_\phi$ from $t=0$ to $t=1$ with standard numerical ODE solvers.

\subsection{The Geometry of Complex-Valued Fields: Cylinder vs. Polar Singularity}
For complex-valued physical fields, such as Magnetic Resonance Imaging (MRI) or audio spectrograms, each spatial pixel has an amplitude $A \in [0, \infty)$ and a phase $q = e^{i\theta} \in S^1$. When local coordinates are needed, the phase is
represented by the principal angle $\theta \in [-\pi, \pi)$ and treated as a choice of representatives for $S^1 \cong \mathbb{R} / 2\pi\mathbb{Z}$, rather than a homeomorphic copy of it. Writing $\mathbb{R}^+ = (0, \infty)$, the punctured plane is
diffeomorphic to the open cylinder $\mathbb{C}^* = \mathbb{C} \setminus \{0\} \cong \mathbb{R}^+ \times S^1$. CyFM operates on its closure $[0, \infty) \times S^1$.

The standard polar coordinates on $\mathbb{C}^*$ express the Euclidean metric as $g_{\mathrm{polar}} = dA^2 + A^2 d\theta^2$, which is flat for $A > 0$, and nonsingular in its domain. What degenerates is its attempted extension to $A = 0$: the differential drops rank there, and the entire circle of phases collapses to a single point. In contrast, by equipping the single-pixel signal domain with a decoupled product metric, we define the local
  \textbf{cylindrical manifold}:
    \begin{equation}
        \mathcal{M}_{\text{local}} = [0, \infty) \times S^1, \quad g_{\text{local}} = dA^2 + d\theta^2.
    \end{equation}
Unlike the polar metric, $g_{\text{local}}$ has constant coefficients and remains non-degenerate throughout the boundary $A = 0$. Both metrics have Riemann curvature tensor equal to zero ($R \equiv 0$), which means that both are flat, so curvature is not what separates them. Instead, $g_{\text{local}}$ differs in that the circle $A = \mathrm{const}$ has length $2\pi$ at every amplitude rather than $2\pi A$ Appendix~\ref{sec:app_metric_choice} derives these circumferences and the polar pullback metric, making CyFM metric substitution explicit. In addition, the non-trivial fundamental group ($\pi_1(\mathcal{M}_{\text{local}}) = \mathbb{Z}$) requires phase to be handled modulo $2\pi$. Because the metric is decoupled, the geodesics on $\mathcal{M}_{\text{local}}$ separate into independent paths: linear interpolation in amplitude on $[0, \infty)$ and minimal circular arcs on $S^1$. Furthermore, the cylinder is parallelisable: representing points as $(A, q)$ with $q \in S^1 \subset \mathbb{C}$, the vector fields $E_A = (1, 0)$ and $E_\theta = (0, i q)$ form a global orthonormal frame, preserving vector components under parallel transport without needing multiple coordinate charts. The angle $\theta$ remains a local coordinate with a periodic jump at $\pm\pi$, which Section~\ref{sec:method} handles by providing the network with a trigonometric phase embedding $(\cos\theta, \sin\theta)$.

\subsection{Destructive Midpoint Attenuation in Euclidean Interpolation}
\label{sec:midpoint_attenuation}
When complex signals are processed as two-channel real tensors $(\mathrm{Re}(z), \mathrm{Im}(z))$, straight Euclidean paths interpolate between states $z_0 = A_0 e^{i\theta_0}$ and $z_1 = A_1 e^{i\theta_1}$. By the triangle inequality, the trajectory amplitude satisfies $|z(t)| = |(1-t)z_0 + t z_1| \le (1-t) A_0 + t A_1$, with strict inequality for all $t \in (0, 1)$ whenever $A_0, A_1 > 0$ and $\theta_0 \not\equiv \theta_1 \pmod{2\pi}$. At the temporal midpoint $t = 0.5$, the amplitude is:
\begin{equation}
    |z(0.5)| = \frac{1}{2} \sqrt{A_0^2 + A_1^2 + 2 A_0 A_1 \cos(\Delta\theta)}, \quad \Delta\theta = \theta_1 - \theta_0.
\end{equation}
For equal endpoint amplitudes $A_0 = A_1 = A$, this simplifies to $|z(0.5)| = A |\cos(\Delta\theta / 2)|$.

Under an independent coupling with uniform phases, the wrapped phase difference $\Delta\theta$ is uniformly distributed on $(-\pi, \pi]$. Averaging over this distribution yields an expected midpoint amplitude of $2/\pi$ of the endpoint amplitude: an inherent attenuation of $1 - 2/\pi \approx 36.3\%$ that no neural network capacity can remove, as it is a geometric property of the path rather than of the model. Appendix~\ref{sec:app_proofs} gives the derivation in Proposition~\ref{prop:attenuation_proof}.

For $A_0, A_1 > 0$, exactly antipodal endpoints are the only ones whose path meets the origin, occurring only at $t = A_0 / (A_0 + A_1)$ where the phase is undefined. In continuous phase distributions, this occurs with probability zero. A vanishing endpoint yields a radial chord whose phase never turns. The practical failure mode is the near-antipodal case at comparable amplitudes, where paths pass arbitrarily close to zero: the amplitude is heavily attenuated and the induced angular velocity becomes arbitrarily large. The $36.3\%$ expectation above applies to independent coupling; minibatch OT changes which endpoints are paired and hence their phase differences. Table~\ref{tab:bridge_ablation} reports a related path-level measure across synthetic, knee MRI, and speech data: the fraction of Cartesian bridges whose peak angular velocity exceeds $\pi$ falls under minibatch OT, although extreme turns remain common for speech. The Cartesian path thus depends on the coupling, whereas the cylindrical bound $|u_\theta| \le \pi$ holds for every pair of endpoints under any coupling.
\section{Cylindrical Flow Matching}
\label{sec:method}

Building on the cylindrical geometry in Section~\ref{sec:preliminaries}, we formulate \textbf{Cylindrical Flow Matching (CyFM)} for an $H\times W$ complex field on the product manifold $\mathcal{M}^{H \times W} = ([0, \infty) \times S^1)^{H \times W}$ with metric $g = \sum_{u, v} (\mathrm{d}A_{u,v}^2 + \mathrm{d}\theta_{u,v}^2)$. This metric makes each conditional geodesic computable pixelwise: amplitude interpolates linearly and phase follows a shortest circular arc, so every angular target satisfies $|u_{\theta,uv}| \le \pi$. We pair prior and data samples as complete fields through joint minibatch Optimal Transport, keeping each field's amplitude--phase and spatial dependencies intact. The neural backbone then uses the full field to learn the resulting vector field.

\subsection{Geodesic Probability Paths and Bounded Target Velocities}
Given a prior sample $x_0 = (A_0, \theta_0) \sim p_0$ from a tractable base distribution $p_0$ and a target field $x_1 = (A_1, \theta_1) \sim p_1$, the geodesic flow path on $\mathcal{M}$ interpolates amplitude along $[0, \infty)$ and phase along $S^1$:
\begin{equation}
    A_t = (1 - t) A_0 + t A_1, \quad \theta_t = (\theta_0 + t \, u_\theta) \pmod{2\pi},
\end{equation}
with target velocities $(u_A, u_\theta) \in T_{(A_t, \theta_t)}\mathcal{M}$ given by:
\begin{equation}
    u_A = A_1 - A_0, \quad u_\theta = \operatorname{atan2}\big(\sin(\theta_1 - \theta_0),\, \cos(\theta_1 - \theta_0)\big).
\end{equation}
The angular velocity $u_\theta$ is the Riemannian logarithmic map $\log_{\theta_0}^{S^1}(\theta_1)$, which yields the shortest circular arc displacement. This shortest path is unique, except at exact antipodal points ($|\Delta\theta| = \pi$), where the sign convention of $\operatorname{atan2}$ deterministically breaks the directional tie.

\begin{lemma}[Bounded Angular Velocity]
\label{lem:bounded_velocity}
For any pair $(x_0, x_1) \in \mathcal{M} \times \mathcal{M}$, the target angular velocity under CyFM is globally bounded: $|u_\theta| \le \pi$. In contrast, wherever $z_t \ne 0$, the induced angular velocity under Cartesian flow $z_t = (1-t)z_0 + t z_1$ is:
\begin{equation}
    \dot{\theta}_{\mathrm{Cart}} = \frac{x_t \dot{y}_t - y_t \dot{x}_t}{|z_t|^2} = \frac{A_0 A_1 \sin(\theta_1 - \theta_0)}{|z_t|^2}.
\end{equation}
Because the numerator remains constant along the chord, $\dot{\theta}_{\mathrm{Cart}}$ diverges as the path approaches the origin ($|z_t| \to 0$). For equal endpoint amplitudes, its peak is $2|\tan(\Delta\theta/2)|$. Appendix~\ref{sec:app_proofs} derives this expression in Corollary~\ref{cor:peak_equal_amplitude} and its index-$1$ tail under uniform phase differences in Corollary~\ref{cor:pareto_tail}.
\end{lemma}
This bound strictly constrains the regression target. The learned network output is not architecturally forced to respect it. Similarly, while the target probability paths remain in $[0, \infty)$, the integrated numerical trajectories could cross the boundary $A = 0$. We project negative amplitudes back to zero during sampling. Appendix~\ref{sec:app_boundary} specifies this projection and reports that it was needed in fewer than $0.003\%$ of solver updates across all runs.

\subsection{Joint Optimal Transport Coupling on the Cylinder}
\label{sec:ot_coupling}
To minimize path crossings, CyFM couples training batches via exact minibatch Optimal Transport \citep{pooladian2023multisample, tong2023improving} computed \emph{jointly over whole fields} in the cylindrical metric. For prior fields $\{x_0^i\}_{i=1}^{B}$ and target fields $\{x_1^j\}_{j=1}^{B}$, the pairwise transport cost is:
\begin{equation}
    \label{eq:ot_cost}
    c_{ij} = \frac{1}{2HW} \sum_{u=1}^{H} \sum_{v=1}^{W} \Big[ \big(A^{i}_{0,uv} - A^{j}_{1,uv}\big)^2 + d_{S^1}\big(\theta^{i}_{0,uv}, \theta^{j}_{1,uv}\big)^2 \Big],
\end{equation}
where $d_{S^1}(\theta_0, \theta_1) = |\operatorname{atan2}(\sin(\theta_1 - \theta_0), \cos(\theta_1 - \theta_0))| \le \pi$ is the shortest circular geodesic distance. The optimal assignment $\sigma^* = \arg\min_{\sigma \in \mathcal{S}_B} \sum_{i=1}^B c_{i, \sigma(i)}$ is solved via linear sum assignment, and geodesic bridges are constructed between $x_0^i$ and $x_1^{\sigma^*(i)}$. Crucially, coupling is joint across amplitude, phase, and pixels: factorising transport per-coordinate or per-patch destroys the joint data distribution (see Section~\ref{sec:exp_factorized}).

\subsection{Continuous Trigonometric Embedding and Objective}
To avoid periodic jump discontinuities of the scalar angles at $\pm\pi$, the inputs are continuously embedded per pixel via:
\begin{equation}
    \mathbf{e}(A, \theta) = \big(A, \cos\theta, \sin\theta\big) \in \mathbb{R}^3.
\end{equation}
The neural backbone takes $\mathbf{e}(A_t, \theta_t)$ with time embedding $t$ and predicts the tangent vector field $v_\phi(t, \mathbf{e}) = (v_A, v_\theta) \in T_{(A_t, \theta_t)}\mathcal{M}$. Under $g_{\mathrm{local}} = \mathrm{d}A^2 + \mathrm{d}\theta^2$, Riemannian Flow Matching reduces to the unweighted squared error:
\begin{equation}
    \label{eq:loss}
    \mathcal{L}_{\mathrm{CyFM}}(\phi) = \mathbb{E}_{t, x_0, x_1} \Big[ (v_A - u_A)^2 + (v_\theta - u_\theta)^2 \Big].
\end{equation}
Because the angular loss compares tangent vectors at the same point, circular wrapping is already handled by $u_\theta$. Adopting a $\mathcal{L}_1$ loss or weighting phase by target amplitude ($A_1/\bar{A}_1$) breaks the flow-matching fixed point (recovering conditional medians or biasing marginal vector fields). Empirical comparisons are detailed in Appendix Table~\ref{tab:unet_loss_protocols}.

\clearpage
\section{Experiments}
\label{sec:experiments}

We organise the evaluation around five complementary investigations. Through analytical probability paths, we show that the cylindrical target eliminates origin singularities (Table~\ref{tab:bridge_ablation}). We then benchmark this geometry with U-Net architectures and Optimal Transport across synthetic, audio, and physical MRI wavefields (Figure~\ref{fig:convergence}, Table~\ref{tab:mri_64}). Across spatial dimensions, we tested the robustness of scaling and the persistence of joint OT benefits, before showing that factorizing the coupling destroys the joint target distribution (Table~\ref{tab:factorized}). Finally, we examine scaling limits on full-resolution clinical matrices.

\subsection{Analytical Bridge Ablation: Isolating the Geometry}
\label{sec:exp_ablation}

To isolate the effect of our geometric formulation without the confounding variables of neural network capacity or empirical coupling variance, we first evaluate exact analytical probability paths under each metric. By computing theoretical trajectories $z_t$ between samples drawn from a uniform phase prior and each complex target, we directly measure the angular velocities predicted in Section~\ref{sec:method}. Because the numerator of the Cartesian angular velocity $\dot{\theta}(t) = \operatorname{Im}(\bar{z}_0 z_1) / |z_t|^2$ is invariant along straight chords, the peak angular velocity is governed exclusively by the minimum distance to the origin.

\begin{table}[t]
    \centering
    \caption{\textbf{Analytical path geometry across complex data sources.} Fraction of $10^6$ exact probability bridges (no network) with peak angular velocity $|\dot{\theta}| > \pi$. \textbf{Paths} treats all coordinates equally; \textbf{Energy} weights each path by target $|z_1|^2$. Targets: synthetic copula, fastMRI knee CORPD ($320\times320$), and LibriSpeech STFT ($64\times64$), each normalised by peak modulus.}
    \label{tab:bridge_ablation}
    \vspace{1mm}
    \begin{tabular}{llcc}
        \toprule
        \textbf{Target} & \textbf{Geometry / coupling} & \textbf{Paths $> \pi$ $\downarrow$} & \textbf{Energy $> \pi$ $\downarrow$} \\
        \midrule
        Synthetic
          & Cartesian, independent  & 46.0\% & 42.6\% \\
          & Cartesian, minibatch OT & 1.0\%  & 0.1\%  \\
          & \textbf{Cylindrical (ours)} & \textbf{0.0\%} & \textbf{0.0\%} \\
        \midrule
        Knee MRI
          & Cartesian, independent  & 56.4\% & 48.3\% \\
          & Cartesian, minibatch OT & 13.1\% & 0.1\%  \\
          & \textbf{Cylindrical (ours)} & \textbf{0.0\%} & \textbf{0.0\%} \\
        \midrule
        Speech STFT
          & Cartesian, independent  & 94.8\% & 49.4\% \\
          & Cartesian, minibatch OT & 81.9\% & 5.1\%  \\
          & \textbf{Cylindrical (ours)} & \textbf{0.0\%} & \textbf{0.0\%} \\
        \bottomrule
    \end{tabular}
    \vspace{1mm}
\end{table}

Table~\ref{tab:bridge_ablation} reports peak angular velocities along $10^6$ sampled endpoint pairs (seed $0$), evaluated in closed form (Proposition~\ref{thm:angular_divergence}). Across all three targets, Cartesian bridges exhibit a heavy-tailed Pareto power law (index $\approx 1.0$, Corollary~\ref{cor:pareto_tail}) with peak velocities diverging near the origin. Under independent coupling, nearly half to over $90\%$ of Cartesian paths exceed $\pi$ rad per unit time. Although minibatch OT partially suppresses these extreme turns for scalar pairs, this mitigation diminishes with dimension (Section~\ref{sec:exp_scaling}). In stark contrast, cylindrical bridges are bounded by $|u_\theta| \le \pi$ by construction across all domains and couplings. Weighting paths by target energy $|z_1|^2$ reinstates the planar metric factor $A^2$ from $\mathrm{d}A^2 + A^2\mathrm{d}\theta^2$, yet even under energy weighting, $42.6\text{--}49.4\%$ of Cartesian signal energy turns faster than $\pi$ under independent coupling, compared to $0.0\%$ on the cylinder. This bounded target provides an inherently more stable regression objective for neural networks.

\subsection{Spatial Field Synthesis: U-Net Evaluation}
\label{sec:exp_unet_baseline}

While analytical bridges characterise regression targets in isolation, practical generative modeling relies on a neural network regressing this vector field.

\textbf{Data and protocol.} Three domains normalised by peak modulus: synthetic copula fields, LibriSpeech speech STFT, and fastMRI knee wavefields, all at $64\times64$. Every arm trains an identical U-Net for 40 epochs from a uniform prior, evaluated across 5 seeds with Heun integration ($k$ steps cost $2k-1$ evaluations). In addition to pooled sliced $W_2$, we evaluate lag-one autocorrelation gaps on amplitude and circular phase to capture spatial structure (Section~\ref{sec:exp_factorized}). A difference is called \emph{separated} when every seed of one arm lies strictly below every seed of the other. Protocol details, cohort sizes, and $k^\star$ definitions are provided in Appendix~\ref{sec:app_protocol}.

Figure~\ref{fig:convergence} displays synthesis error across solver steps on synthetic fields and speech STFT ($64\times64$). Across both continuous domains, CyFM with joint OT achieves strictly lower error at every $k \le 8$, visible as non-overlapping seed envelopes against the best Cartesian arm. On synthetic fields, CyFM cuts single-step error threefold before converging to parity with Cartesian flows by $k=100$ ($p=0.15$). On speech STFT, CyFM reaches its error floor in just $4$ steps ($k^\star = 4$) compared to $100$ steps for Cartesian baselines, that being a $25\times$ acceleration to target fidelity.

\begin{figure}[t]
    \centering
    \includegraphics[width=\linewidth]{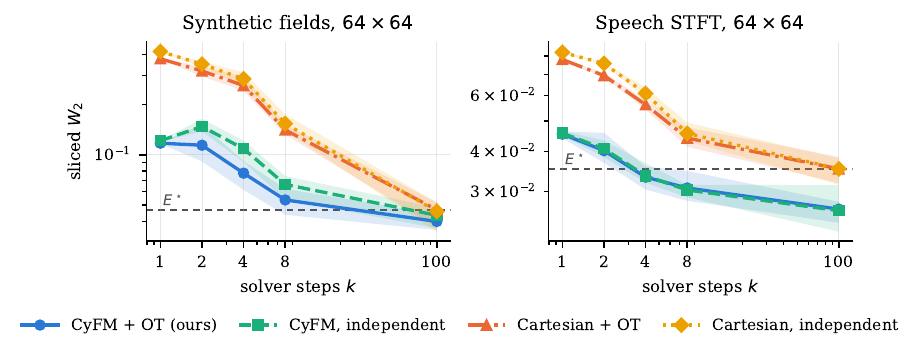}
    \caption{\textbf{Error against solver steps.} Sliced $W_2$, mean over 5 seeds, with shaded bands spanning minimum to maximum seed (non-touching bands indicate separated seeds). See Table~\ref{tab:unet_baseline} for exact values.}
    \label{fig:convergence}
\end{figure}

\textbf{Physical wavefields: spatial structure beyond marginals on Knee MRI.} Sliced $W_2$ measures 2D marginal distributions; physical knee wavefields additionally carry anatomical correlations and smooth phase fields. Table~\ref{tab:mri_64} reports pooled marginal distance alongside measures of spatial arrangement.

\begin{table}[!h]
    \centering
    \caption{\textbf{Physical field synthesis on fastMRI Knee ($64\times64$).} Evaluation on held-out test slices across solver steps $k$ (mean over 5 seeds). Asterisks ($^\star$) mark strict seed separation across arms ($p=0.008$, two-sided Mann--Whitney). The score-based baseline is a variance-exploding SDE \citep{song2020score} matched in training cohort, evaluation budget, and seeds.}
    \label{tab:mri_64}
    \vspace{1mm}
    \resizebox{\linewidth}{!}{
    \begin{tabular}{llccccc}
        \toprule
        \textbf{Metric} & \textbf{Arm} & $k=1$ & $k=2$ & $k=4$ & $k=8$ & $k=100$ \\
        \midrule
        \multicolumn{7}{l}{\textit{Sliced $W_2$ ($\downarrow$) -- pixelwise marginal distribution}} \\
        & \textbf{Cylindrical + OT (ours)} & \textbf{0.0593}$^\star$ & 0.1200 & 0.0937 & 0.0735 & 0.0722 \\
        & Cylindrical, independent         & 0.0604 & 0.1338 & 0.1064 & 0.0756 & 0.0677 \\
        & Cartesian + OT                   & 0.1062 & \textbf{0.1120}$^\star$ & \textbf{0.0872} & 0.0614 & 0.0617 \\
        & Cartesian, independent           & 0.1518 & 0.1432 & 0.1054 & \textbf{0.0582}$^\star$ & \textbf{0.0563} \\
        & Complex diffusion (VE-SDE)       & 53.18 & 53.16 & 429.50 & 67.44 & 0.0771 \\
        \midrule
        \multicolumn{7}{l}{\textit{Spatial Lag-1 Gap ($\downarrow$) -- amplitude texture autocorrelation}} \\
        & \textbf{Cylindrical + OT (ours)} & 0.2971 & 0.1598 & \textbf{0.0339} & \textbf{0.0235}$^\star$ & \textbf{0.0153}$^\star$ \\
        & Cartesian, independent           & \textbf{0.0758}$^\star$ & \textbf{0.0495}$^\star$ & 0.0561 & 0.0404 & 0.0305 \\
        \midrule
        \multicolumn{7}{l}{\textit{Circular Phase Lag-1 Gap ($\downarrow$) -- phase spatial coherence}} \\
        & \textbf{Cylindrical + OT (ours)} & 0.4530 & 0.3867 & 0.2771 & \textbf{0.1517} & \textbf{0.1057}$^\star$ \\
        & Cartesian, independent           & \textbf{0.0922}$^\star$ & \textbf{0.1503}$^\star$ & \textbf{0.1483}$^\star$ & 0.1712 & 0.1712 \\
        \midrule
        \multicolumn{7}{l}{\textit{Radial Spectrum Gap ($\downarrow$) -- frequency power distribution}} \\
        & \textbf{Cylindrical + OT (ours)} & 0.4952 & 0.4195 & \textbf{0.1379} & \textbf{0.1233} & 0.1043 \\
        & Cartesian, independent           & \textbf{0.2936}$^\star$ & \textbf{0.1340}$^\star$ & 0.1725 & 0.1429 & \textbf{0.0976} \\
        \bottomrule
    \end{tabular}
    }
    \vspace{1mm}
\end{table}

In a single step ($k=1$), CyFM cuts the marginal error pooled by $1.8\times$ on the best Cartesian arm ($0.0593$ vs $0.1062$), with all five seeds strictly separated. Although Cartesian flows achieve a lower pooled marginal error for $k \ge 2$, the ranking reverses on spatial structure from $k=4$ onward (Table~\ref{tab:mri_64}). On amplitude texture autocorrelation ($\texttt{spatial\_lag1}$), CyFM outperforms Cartesian flows across all $k \ge 4$, achieving a $2.0\times$ reduction at $k=100$ with separated seeds from $k=8$ (at $k=4$, the cohort reproduces mean ordering, with full separation holding from $k=8$). On circular phase spatial coherence ($\texttt{phase\_lag1}$), CyFM attains a sharper phase $1.6\times$ at $k=100$, separated in all five seeds. The OT-coupled Cartesian baseline closely tracks independent Cartesian, showing that this structural reversal stems from the geometry rather than the coupling.

At this matrix resolution, central k-space truncation disperses tissue signal across the field of view, eliminating the exact zero-valued background pixels observed at $320\times320$. On these continuous wavefields, the unweighted metric $\mathrm{d}A^2 + \mathrm{d}\theta^2$ strictly outperforms an amplitude-weighted phase term on both spatial measures at $k=100$: $0.0153$ against $0.0219$ on amplitude texture and $0.1057$ against $0.1656$ on phase coherence.

\textbf{A score-based reference, and the transition at $k=2$.} Two parameterisations of flow matching do not constitute a full baseline family, so Table~\ref{tab:mri_64} also reports a variance-exploding SDE \citep{song2020score} matched in cohort, budget, and seeds. Across few steps ($k \le 8$), the reverse SDE fails to contract onto the data manifold, leaving errors orders of magnitude above data scale. Scored at matched budget it requires $100$ steps to reach $0.0771$, which every flow arm beats at $k=100$ and three beat at $k=8$; in its own $1000$-step regime it reaches $0.0881$. On a dense step grid, CyFM reaches matched target fidelity $E^\star = 0.0771$ at $k^\star = 1$, whereas Cartesian arms require $k^\star = 8$.

The cylindrical marginal error curve is non-monotone on this domain, rising from $k=1$ to $k=2$ before falling. Five candidate mechanisms are ruled out experimentally in Appendix~\ref{sec:app_k2} (undertraining, zero background, velocity bound violations, dimension, and bimodal amplitude laws). At $k=1$, a single Euler step transports amplitude almost completely while phase barely evolves; a comprehensive theoretical account remains an open question.

\subsection{Dimensionality Scaling and the Joint OT Coupling}
\label{sec:exp_scaling}

The transport-cost reduction that minibatch OT achieves over independent pairing collapses with dimension ($86\%$ on scalar pairs, $21\%$ on $8\times8$, and $3\%$ on $64\times64$ fields under the cylindrical metric). To test whether generative benefits collapse as well, we ablate U-Net models across dimensions ($16\times16$, $32\times32$, $64\times64$) with and without joint coupling (Appendix Table~\ref{tab:unet_scaling}).

Generative gains do not mirror this collapse. Across all resolutions, joint cylindrical OT improves few-step error ($3\text{--}60\%$ reduction at $k \le 4$, with seeds separated in 6 of 9 cases; exceptions are $k=1$ at $16\times16$ and $k=2, 4$ at $64\times64$). CyFM achieves lower error than the best Cartesian baseline at all $k \le 8$ across all three resolutions with separated seeds. At $k=100$, differences are not statistically significant ($p = 0.69$ at $64\times64$, best arms $0.035$ vs $0.032$; Appendix Table~\ref{tab:unet_loss_protocols}). OT also benefits Cartesian flows at low resolution, improving $16\times16$ $k=1$ error from $0.442$ to $0.227$.

\subsection{The Factorised Coupling Trap}
\label{sec:exp_factorized}

Because the cylindrical metric is separable across coordinates and pixels, it is tempting to decouple Optimal Transport per coordinate or spatial patch. The factorised minimum is always a lower bound on the joint one, attaining equality only when both distributions are product measures or coincide degenerately. Table~\ref{tab:factorized} demonstrates the cost of this shortcut on the spiral copula target ($n = 1024$, 8 seeds).

\begin{table}[!h]
    \centering
    \caption{\textbf{The Factorised Coupling Trap} (no network). Coupling amplitude and phase independently preserves both marginals exactly but erases their dependence, and reports a transport cost below the true optimum. Spiral copula target, $n = 1024$, mean over 8 seeds.}
    \label{tab:factorized}
    \vspace{1mm}
    \begin{tabular}{cccccc}
        \toprule
        $\rho$ & \textbf{Corr.\ target} & \textbf{Corr.\ coupled} & \textbf{KS (marginals)} & \textbf{Cost, factorised} & \textbf{Cost, joint OT} \\
        \midrule
        0.0 & 0.024 & 0.042 & 0 & 0.200 & 0.206 \\
        0.5 & 0.310 & 0.042 & 0 & 0.200 & 0.225 \\
        1.0 & 0.867 & 0.040 & 0 & 0.200 & 0.385 \\
        \bottomrule
    \end{tabular}
\end{table}

Coupling amplitude and phase independently preserves 1D marginals bit-identically (Kolmogorov--Smirnov distance 0), yet circular-linear correlation collapses to $\approx 0.04$ across all true target correlations $\rho$, and transport cost reports an artificial $0.20$ while true joint cost reaches $0.385$. The spatial analogue is identical: coupling $16\times16$ patches independently on $64\times64$ fields preserves intra-patch correlations ($0.979$) and pooled sliced $W_2$ ($0.000$), but collapses lag-1 autocorrelation across patch seams from $0.979$ to $0.044$. Factorised coupling transports to a reassembled product law rather than the true data distribution; neural network capacity cannot remedy an objective fitted to the wrong target.

\subsection{Scope and Scaling to Full-Resolution Clinical Matrices}
\label{sec:exp_scaling_320}

Scaling generative flow matching from $64\times64$ crops to full-resolution clinical MRI matrices ($320\times320$, $102{,}400$ complex dimensions) introduces distinct physical and geometric challenges. In Appendix~\ref{sec:app_mri_scaling}, we report the complete 5-seed benchmark comparing CyFM and Cartesian flows at $320\times320$ across sliced $W_2$, spatial lag-1, phase lag-1, and radial spectral metrics.

At a single step ($k=1$), CyFM preserves its few-step distributional advantage: it attains a sliced $W_2$ of $0.0561$ compared to $0.1203$ for Cartesian+OT ($2.1\times$ lower error, separated 5/5 seeds). However, at full resolution, Cartesian interpolation outperforms the cylindrical geometry on every spatial measure across all step counts. This cannot be attributed to the collapse of minibatch OT, as a cylindrical arm trained with independent coupling behaves identically ($0.3926$ vs $0.3882$ on the radial spectrum gap). As discussed in Appendix~\ref{sec:app_mri_scaling}, a candidate mechanism supported by empirical weighting ablations is the presence of non-measured background air ($16.7\%$ exactly zero coefficients), which the decoupled metric weights equally but the inherited Euclidean metric suppresses via $A^2$. Principled handling of unmeasured background in uncropped clinical acquisitions remains an important open challenge.

\section{Conclusion}
\label{sec:conclusion}

In this empirical study, we compared Euclidean and cylindrical geometries for complex-valued generative Flow Matching across synthetic fields, speech audio STFT spectrograms, and fastMRI knee wavefields. On exact probability paths, Cartesian interpolation produces a heavy-tailed angular velocity near the origin, while mapping complex fields onto the decoupled cylindrical manifold ($[0, \infty) \times S^1$) bounds the target angular velocity by $\pi$ by construction. Across all three domains at $64\times64$, Cylindrical Flow Matching (CyFM) with joint Optimal Transport achieves a lower error than the Cartesian baselines at a single step, a $1.8\times$--$3.0\times$ reduction without distillation. On synthetic fields and on speech that advantage extends to every $k \le 8$ with separated seeds; on knee MRI it is confined to $k=1$ on the pooled metric, and reappears from $k=4$ on measures of spatial structure, where CyFM reproduces the reference's amplitude texture twice as closely. Against a variance-exploding score-based baseline trained on the same cohort and scored at matched model calls, the single-step gap is three orders of magnitude, and that baseline needs a hundred steps to reach an error every flow arm beats there.

Our analysis also revealed critical structural insights when scaling flows to high dimensions. While the transport-cost reduction of minibatch OT shrinks on high-dimensional fields, joint cylindrical coupling preserves significant few-step error reductions across all resolutions. In contrast, factorising couplings across coordinates or patches quietly destroys the joint target distribution (the Factorised Coupling Trap). Finally, at the full clinical matrix ($320\times320$) the single-step advantage persists ($2.1\times$ on the pooled metric), but the Cartesian arm is closer to the reference on every spatial measure at every step count. The cause is open, and one candidate is already excluded: it is not the collapse of minibatch OT with dimension, because the cylindrical arm trained with no coupling at all behaves the same. We therefore report it as a scaling limit of a single global coordinate representation rather than of the coupling (Appendix~\ref{sec:app_mri_scaling}), and identify it as the open question this geometry now faces.

\newpage
\subsubsection*{AI Use Statement}
In this research, generative and agentic AI tools were used to assist the authors in software engineering (modular pipeline implementation in \texttt{cyfm}, codebase refactoring, and automated test suite development) as well as writing (language polishing, conciseness editing to meet page limits, and \LaTeX{} table formatting). All numerical code, experimental pipelines, and text were reviewed, tested, and verified by the authors, who take full responsibility for the contents of this paper.

\subsubsection*{Ethics Statement}
This research exclusively utilizes open scientific benchmark datasets. The fastMRI knee dataset is accessed under the NYU fastMRI Data Sharing Agreement with patient health identifiers de-identified and anonymized. The LibriSpeech corpus is distributed under the Creative Commons Attribution 4.0 International license. This study focuses on mathematical and algorithmic foundations of generative flow matching and presents no direct clinical deployment or societal harm risks.

\subsubsection*{Reproducibility Statement}
All primary quantitative results in this paper are deterministic and fully reproducible. The complete code implementation, model definitions, ODE solvers, and Hydra configurations are included in the supplementary materials and released under the MIT license. For synthetic experiments, data generation protocols and explicit RNG seeds are provided. For real-world benchmarks, the preprocessing pipelines, coil-combination routines, and split manifests for fastMRI knee acquisitions and LibriSpeech STFT representations are documented in Appendix~\ref{sec:app_protocol} and provided under \texttt{scripts/data/}. All claims in Tables~\ref{tab:bridge_ablation}--\ref{tab:factorized} and Appendix~\ref{sec:app_mri_scaling} are backed by archived experiment artifacts with exact command-line provenance in \texttt{docs/reproduce/}\allowbreak\texttt{paper\_results/} and can be re-verified using the automated test suite (\texttt{pytest tests/test\_paper\_results.py}) or the check runner (\texttt{python -m reproducibility.run\_all}).

\begin{ack}
Calculations have been carried out in Wroc{\l}aw Centre for Networking and Supercomputing (\url{http://www.wcss.pl}/), grant No.~hpc-danbor2008-1756464546.
\end{ack}

\bibliography{references}
\bibliographystyle{iclr2027_conference}

\newpage
\appendix
\section{The Cylinder is a Metric Substitution, not an Inherited Geometry}
\label{sec:app_metric_choice}

\begin{lemma}[Amplitude--phase chart and its pullback metric]
\label{lem:chart_pullback}
Let $\mathbb{C}^* = \mathbb{C} \setminus \{0\}$ and let
\begin{equation}
    \Phi : \mathbb{C}^* \to (0, \infty) \times S^1, \qquad
    \Phi(z) = \left( |z|, \; z / |z| \right),
\end{equation}
with inverse $\Phi^{-1}(A, \theta) = (A\cos\theta, A\sin\theta)$. Then $\Phi$ is a diffeomorphism, and the pullback of the Euclidean metric $g_{\mathrm{euc}} = dx^2 + dy^2$ under $\Phi^{-1}$ is
\begin{equation}
    (\Phi^{-1})^* g_{\mathrm{euc}} = dA^2 + A^2 d\theta^2 .
\end{equation}
\end{lemma}
\begin{proof}
Both $\Phi$ and $\Phi^{-1}$ are smooth on their domains and are mutually inverse by construction, so $\Phi$ is a diffeomorphism. Differentiating $\Phi^{-1}$,
\begin{align}
    dx &= \cos\theta \, dA - A \sin\theta \, d\theta, \\
    dy &= \sin\theta \, dA + A \cos\theta \, d\theta .
\end{align}
Squaring and summing, the cross terms $\mp 2A\sin\theta\cos\theta \, dA \, d\theta$ cancel and the remaining terms group by $\cos^2\theta + \sin^2\theta = 1$:
\begin{equation}
    dx^2 + dy^2 = dA^2 + A^2 d\theta^2 . \qedhere
\end{equation}
\end{proof}

CyFM does not use this metric. It replaces it with the decoupled product metric $g_{\mathrm{cyl}} = dA^2 + d\theta^2$, that is, it deletes the factor $A^2$. The map $\Phi$ is therefore a diffeomorphism but \emph{not} an isometry between $(\mathbb{C}^*, g_{\mathrm{euc}})$ and $((0,\infty) \times S^1, g_{\mathrm{cyl}})$. This substitution, and not the identification of the domain, is the modelling decision this paper evaluates.

\begin{remark}[Flatness is not the distinguishing property]
\label{rem:both_flat}
Both metrics have identically vanishing Riemann curvature: $g_{\mathrm{euc}}$ because $\mathbb{C}^*$ is an open subset of the plane, and $g_{\mathrm{cyl}}$ because it is a product of one-dimensional factors. Each is therefore locally isometric to $\mathbb{R}^2$, and no local computation separates them. They are not globally isometric: under $g_{\mathrm{cyl}}$ the curves $A = \mathrm{const}$ are closed geodesics, whereas the geodesics of $g_{\mathrm{euc}}$ are straight line segments and none is closed. The non-trivial fundamental group $\pi_1 = \mathbb{Z}$ is a property of the punctured domain, shared by both, and is a cost the model pays through the $2\pi$ wrap rather than a benefit it gains.
\end{remark}

\begin{corollary}[Circumference]
\label{cor:circumference}
The circle $\{A = r\}$ has length $2\pi r$ under $g_{\mathrm{euc}}$ and length $2\pi$ under $g_{\mathrm{cyl}}$, for every $r > 0$.
\end{corollary}

This single difference accounts for both of the paper's velocity statements. Under $g_{\mathrm{euc}}$ the circumference available for a phase change shrinks to zero as $r \to 0$, so a path carrying a fixed phase difference through a neighbourhood of the origin must traverse it at ever greater angular speed; this is the $|c| / |z_t|^2$ growth of Proposition~\ref{thm:angular_divergence}. Under $g_{\mathrm{cyl}}$ the circumference is constant, so a half-turn costs the same at every amplitude and an angular displacement of $\pi$ always suffices; this is the bound of Lemma~\ref{lem:bounded_velocity}. The two results are the same fact about the factor $A^2$, read at the two ends of its range.

\begin{remark}[Boundary projection at $A = 0$]
\label{sec:app_boundary}
The product manifold $[0, \infty) \times S^1$ has a boundary at $A = 0$. Although target probability paths $A_t = (1-t)A_0 + t A_1$ satisfy $A_t \ge 0$ for all endpoints $A_0, A_1 \ge 0$, discrete numerical ODE integration steps may occasionally overshoot below zero due to discretization error. When this occurs, we project back onto the boundary via $A \leftarrow \max(A, 0)$. In our experiments across all step counts and datasets, this projection was triggered in fewer than $0.003\%$ of integration updates.
\end{remark}

\section{Path-Level Consequences of the Euclidean Metric}
\label{sec:app_proofs}

We present the formal proofs for the statements made in the main text about what the inherited Euclidean metric on $(\mathrm{Re}, \mathrm{Im}) \in \mathbb{R}^2$ implies for probability paths and for the regression target.

Let $z_1 = e^{i\theta_1}$ and $z_2 = e^{i\theta_2}$ be two unit-amplitude phasors on the circle $S^1$, and let $\Delta\theta = \theta_1 - \theta_2 \in [-\pi, \pi]$ denote their phase difference.

\begin{lemma}[Chordal vs. Geodesic Distance]
\label{lem:chord_vs_arc}
The Euclidean (chordal) distance $d_{\mathrm{chord}}$ in the complex plane and the geodesic (arc) distance $d_{S^1}$ on the circle are related by:
\begin{align}
    d_{\mathrm{chord}}(\theta_1, \theta_2) &= \|e^{i\theta_1} - e^{i\theta_2}\|_2 = 2 \left|\sin\left(\frac{\Delta\theta}{2}\right)\right|, \\
    d_{S^1}(\theta_1, \theta_2) &= |\Delta\theta|.
\end{align}
\end{lemma}
\begin{proof}
Expanding the squared Euclidean distance:
\begin{align}
    \|e^{i\theta_1} - e^{i\theta_2}\|_2^2 &= (\cos\theta_1 - \cos\theta_2)^2 + (\sin\theta_1 - \sin\theta_2)^2 \\
    &= (\cos^2\theta_1 + \sin^2\theta_1) + (\cos^2\theta_2 + \sin^2\theta_2) - 2(\cos\theta_1\cos\theta_2 + \sin\theta_1\sin\theta_2) \\
    &= 2 - 2\cos(\theta_1 - \theta_2) = 2(1 - \cos\Delta\theta).
\end{align}
Using the half-angle trigonometric identity $1 - \cos\alpha = 2\sin^2(\alpha/2)$:
\begin{equation}
    \|e^{i\theta_1} - e^{i\theta_2}\|_2^2 = 4\sin^2\left(\frac{\Delta\theta}{2}\right).
\end{equation}
Taking the square root yields $d_{\mathrm{chord}} = 2 \left|\sin\left(\frac{\Delta\theta}{2}\right)\right|$. The geodesic distance on the unit circle is trivially the absolute arc length $|\Delta\theta|$.
\end{proof}

\begin{proposition}[Expected Midpoint Attenuation]
\label{prop:attenuation_proof}
Consider a linear (Euclidean) interpolation between two unit phasors: $z_t^{\mathrm{euc}} = (1-t)e^{i\theta_1} + t e^{i\theta_2}$ for $t \in [0, 1]$. Under a uniform phase difference prior $\Delta\theta \sim \mathcal{U}((-\pi, \pi])$, the expected amplitude at the midpoint $t=0.5$ is exactly $2/\pi$, representing an expected signal attenuation of:
\begin{equation}
    1 - \frac{2}{\pi} \approx 36.338\%.
\end{equation}
\end{proposition}
\begin{proof}
At $t=0.5$, the interpolated signal is:
\begin{equation}
    z_{0.5}^{\mathrm{euc}} = \frac{1}{2} (e^{i\theta_1} + e^{i\theta_2}) = e^{i\frac{\theta_1+\theta_2}{2}} \cos\left(\frac{\Delta\theta}{2}\right).
\end{equation}
The amplitude is $|z_{0.5}^{\mathrm{euc}}| = \left|\cos\left(\frac{\Delta\theta}{2}\right)\right|$. The expected value under a uniform distribution $p(\Delta\theta) = \frac{1}{2\pi}$ is:
\begin{equation}
    \mathbb{E}\left[|z_{0.5}^{\mathrm{euc}}|\right] = \frac{1}{2\pi} \int_{-\pi}^\pi \left|\cos\left(\frac{\Delta\theta}{2}\right)\right| d\Delta\theta.
\end{equation}
Since $\Delta\theta \in [-\pi, \pi]$, the half-angle $\Delta\theta/2 \in [-\pi/2, \pi/2]$. In this domain, the cosine function is non-negative, so the absolute value can be dropped:
\begin{align}
    \mathbb{E}\left[|z_{0.5}^{\mathrm{euc}}|\right] &= \frac{1}{2\pi} \int_{-\pi}^\pi \cos\left(\frac{\Delta\theta}{2}\right) d\Delta\theta \\
    &= \frac{1}{2\pi} \left[ 2\sin\left(\frac{\Delta\theta}{2}\right) \right]_{-\pi}^\pi \\
    &= \frac{1}{2\pi} \left( 2(1) - 2(-1) \right) = \frac{4}{2\pi} = \frac{2}{\pi} \approx 0.6366.
\end{align}
The expected loss in amplitude is therefore $1 - 2/\pi \approx 36.338\%$. By contrast, on the cylindrical manifold, $A_t = (1-t)A_0 + tA_1$, which trivially evaluates to $1.0$ at $t=0.5$, yielding exactly $0\%$ attenuation.
\end{proof}

\begin{proposition}[Vanishing Gradient of the Chordal Penalty]
\label{prop:vanishing_gradient}
Let the chordal Euclidean loss be $\mathcal{L}_{\mathrm{chord}}(\theta) = \frac{1}{2}\|e^{i\theta} - e^{i\theta^*}\|_2^2 = 1 - \cos(\theta - \theta^*)$ and the geodesic loss be $\mathcal{L}_{\mathrm{geod}}(\theta) = \frac{1}{2}(\theta - \theta^*)^2$. As the phase error reaches its maximum $\Delta\theta = |\theta - \theta^*| \to \pi$, the gradient of the chordal penalty vanishes to zero, while the geodesic gradient reaches its supremum.
\end{proposition}
\begin{proof}
Taking the derivative of the chordal loss with respect to the phase $\theta$:
\begin{equation}
    \frac{\partial \mathcal{L}_{\mathrm{chord}}}{\partial \theta} = \frac{d}{d\theta} (1 - \cos(\theta - \theta^*)) = \sin(\theta - \theta^*).
\end{equation}
Taking the limit as $\Delta\theta \to \pi$, we find $\lim_{\Delta\theta \to \pi} |\sin(\Delta\theta)| = \sin(\pi) = 0$. Thus, at the point of maximum possible error (antipodal phase), the network receives exactly zero corrective gradient signal.
Conversely, the geodesic loss derivative is:
\begin{equation}
    \frac{\partial \mathcal{L}_{\mathrm{geod}}}{\partial \theta} = \theta - \theta^*,
\end{equation}
which linearly increases with the error, providing a maximal restoring force of $\pi$ at $\Delta\theta \to \pi$.
\end{proof}

\begin{proposition}[Angular velocity of the Cartesian bridge]
\label{thm:angular_divergence}
Let $z_0, z_1 \in \mathbb{C}$ with $z_0 \ne z_1$, and let $z_t = (1-t) z_0 + t z_1$. Write $c = \operatorname{Im}(\bar{z}_0 z_1)$. At every $t$ with $z_t \ne 0$ the induced angular velocity is
\begin{equation}
    \dot{\theta}_t = \frac{\operatorname{Im}(\bar{z}_t \dot{z}_t)}{|z_t|^2} = \frac{c}{|z_t|^2},
\end{equation}
so the numerator does not depend on $t$, and
\begin{equation}
    \label{eq:peak_angular}
    \sup_{t \in [0,1]} |\dot{\theta}_t| = \frac{|c|}{d^2}, \qquad d = \min_{t \in [0,1]} |z_t| ,
\end{equation}
with the supremum attained at the closest approach to the origin.
\end{proposition}
\begin{proof}
With $z = x + iy$ and $A = |z|$, the map $\theta = \operatorname{atan2}(y, x)$ is differentiable at every point with $A > 0$, and the chain rule gives $\dot{\theta} = (x\dot{y} - y\dot{x}) / A^2 = \operatorname{Im}(\bar{z}\dot{z}) / |z|^2$. Along the bridge $\dot{z}_t = z_1 - z_0$ is constant, and
\begin{align}
    \operatorname{Im}(\bar{z}_t \dot{z}_t)
    &= \operatorname{Im}\big( ((1-t)\bar{z}_0 + t \bar{z}_1)(z_1 - z_0) \big) \\
    &= (1-t) \operatorname{Im}(\bar{z}_0 z_1) - t \operatorname{Im}(\bar{z}_1 z_0) = \operatorname{Im}(\bar{z}_0 z_1) = c,
\end{align}
because $\operatorname{Im}(\bar{z}_0 z_0) = \operatorname{Im}(\bar{z}_1 z_1) = 0$ and $\operatorname{Im}(\bar{z}_1 z_0) = -c$. The numerator being constant, $|\dot{\theta}_t|$ is largest exactly where $|z_t|$ is smallest.
\end{proof}

Two consequences are worth separating from the usual informal statement. First, the bridge does \emph{not} approach the origin arbitrarily closely for fixed endpoints: combining $|\dot{\theta}| = |c| / A^2$ with the Cauchy--Schwarz bound $|\dot{\theta}| \le |z_1 - z_0| / A$ gives $d \ge |c| / |z_1 - z_0|$, which is the distance from the origin to the line through $z_0$ and $z_1$. Second, an upper bound of order $A^{-1}$ on $|\dot{\theta}|$ does not by itself establish divergence, since a purely radial flow has $c = 0$ and $\dot{\theta} \equiv 0$ while satisfying the same bound. What diverges, and the regime the experiments probe, is the peak over a bridge whose endpoints approach antipodal.

\begin{corollary}[Peak angular velocity at equal amplitudes]
\label{cor:peak_equal_amplitude}
If $|z_0| = |z_1| = A > 0$ and $\Delta\theta = \operatorname{wrap}(\theta_1 - \theta_0)$, the closest approach occurs at $t = 1/2$ and
\begin{equation}
    \sup_{t \in [0,1]} |\dot{\theta}_t| = 2 \left| \tan\left( \frac{\Delta\theta}{2} \right) \right| ,
\end{equation}
which is unbounded as $\Delta\theta \to \pm\pi$ and equals the cylindrical rate $|u_\theta| = |\Delta\theta|$ only at $\Delta\theta = 0$.
\end{corollary}
\begin{proof}
$|z_t|^2 = A^2\big(1 - 2t(1-t)(1 - \cos\Delta\theta)\big)$ is minimised at $t = 1/2$, where $|z_{1/2}| = A|\cos(\Delta\theta/2)|$. With $c = A^2 \sin\Delta\theta = 2A^2 \sin(\Delta\theta/2)\cos(\Delta\theta/2)$, equation~\eqref{eq:peak_angular} gives $2|\sin(\Delta\theta/2)| / |\cos(\Delta\theta/2)|$.
\end{proof}

\begin{corollary}[Heavy tail under a uniform phase difference]
\label{cor:pareto_tail}
If in addition $\Delta\theta \sim \mathcal{U}(-\pi, \pi]$, then for $x > 0$
\begin{equation}
    \mathbb{P}\Big( \sup_t |\dot{\theta}_t| > x \Big) = 1 - \frac{2}{\pi} \arctan\left(\frac{x}{2}\right) = \frac{4}{\pi x} + \mathcal{O}(x^{-3}),
\end{equation}
a Pareto tail of index $1$.
\end{corollary}
\begin{proof}
$2|\tan(\Delta\theta/2)| > x$ iff $|\Delta\theta| > 2\arctan(x/2)$, and $|\Delta\theta|$ is uniform on $[0, \pi]$. Expanding $\arctan(x/2) = \pi/2 - 2/x + \mathcal{O}(x^{-3})$ gives the asymptotic form.
\end{proof}

Corollary~\ref{cor:pareto_tail} fixes the amplitudes. When they are random as well, the peak is set by a race between a numerator that vanishes linearly in the smaller amplitude and a squared distance that vanishes quadratically: writing $A_0 = 1$ and $A_1 = a$, both $c$ and $d$ are of order $a |\sin \Delta\theta|$, so $\sup_t |\dot{\theta}_t| = |c| / d^2$ is of order $1 / (a |\sin \Delta\theta|)$ and small amplitudes enlarge the peak rather than shrink it. The index that results depends on the law of the amplitudes near zero and is not fixed by this argument; an exactly vanishing endpoint contributes nothing at all, since $c = 0$ makes the chord radial and its argument never turns. The index near $1$ reported in Table~\ref{tab:bridge_ablation} is therefore an empirical finding rather than a consequence of this corollary. The cylindrical bridge admits no such tail: $\theta_t = \theta_0 + t u_\theta$ turns at the constant rate $u_\theta$, and $|u_\theta| \le \pi$ for every pair of endpoints.

\section{Experimental Protocol in Full}
\label{sec:app_protocol}

The main text states the protocol in the form a reader needs while looking at a table.
This is the rest of it.

\textbf{Cohorts.} The synthetic complex field is drawn from a Gaussian copula with
$\rho = 0.5$ and a correlation length of 4 pixels, $8{,}192$ fields. Speech is
LibriSpeech read speech as complex STFT segments, $64\times64$ bins by frames,
$110{,}099$ segments. The fastMRI cohort is knee coronal proton-density without fat
suppression, coil-combined with ESPIRiT; the main-text block reduces the acquisition
matrix to $64\times64$ by truncating k-space about DC, and
Appendix~\ref{sec:app_mri_scaling} reports the store's own $320\times320$.

\textbf{Training and sampling.} Convolutional U-Nets, 40 epochs, batch 64, five seeds.
Both geometries share the independent uniform prior: amplitude $\mathcal{U}[0,1]$ and
phase $\mathcal{U}[0, 2\pi)$ per pixel. Solvers use Heun integration, so $k$ steps cost
$2k-1$ vector evaluations and step counts are comparable across arms as budgets rather
than as iterations.

\textbf{$k^\star$.} Arms that converge to different asymptotes cannot be compared by
error alone, so $k^\star$ is the first step count at which an arm reaches $E^\star$, the
largest $k=100$ error among the arms of its block. It is the median over seeds, read off
a dense step grid. Fixing $E^\star$ to the block's worst asymptote is what stops the
statistic rewarding an arm for converging to a worse place.

\textbf{Structural measures.} Each reports the gap between a generated batch's lag-one
autocorrelation and the reference cohort's, so a field smoother than the data is
penalised as much as one rougher than it. Phase coherence is circular,
$\mathbb{E}[\cos(\theta_{i+1} - \theta_i)]$ over horizontally adjacent coefficients, and
a generated batch is scored on the share of pairs whose phase the reference actually
measured: $\operatorname{atan2}(0,0)$ returns a placeholder rather than an angle, so a
pair touching an exactly zero amplitude carries no phase information and scoring the
model on it would compare its output against a constant.

\textbf{What the tests support.} A difference is called separated when every seed of one
arm lies strictly below every seed of the other, which for five seeds against five is a
two-sided Mann--Whitney $p = 0.008$. With that many seeds the tests have little power, so
where we report no significant difference that is a failure to detect one and not
evidence of equivalence; an equivalence claim would need a prespecified margin and a test
against it. The headline claim is an intersection over step counts and resolutions, so
every component has to pass on its own and the conjunction needs no multiplicity
correction.

\section{Both Geometries under Both Losses}
\label{sec:app_losses}
Tables~\ref{tab:unet_baseline} and~\ref{tab:unet_scaling} train both geometries with the same objective, the unweighted squared velocity error. Table~\ref{tab:unet_loss_protocols} repeats the grid with the absolute error as well, so that each geometry can also be compared at its best objective. The best of the four variants is selected on the evaluation itself, so the comparison is post-selection and its nominal $p$-values do not hold; we report it descriptively, as a robustness check rather than as a test.

\begin{table}[t]
    \centering
    \caption{\textbf{Dimensionality Scaling.} Generative error (Sliced $W_2$, $\downarrow$) across spatial resolutions, mean over 5 seeds, with the cylindrical model trained with and without the joint OT coupling. All arms use the unweighted squared velocity error. CyFM's few-step advantage holds at all three resolutions, and joint OT improves the cylindrical model at every step count.}
    \label{tab:unet_scaling}
    \vspace{1mm}
    \resizebox{\linewidth}{!}{
    \begin{tabular}{lllccccc}
        \toprule
        \textbf{Resolution} & \textbf{Geometry} & \textbf{Coupling} & \textbf{1 Step} & \textbf{2 Steps} & \textbf{4 Steps} & \textbf{8 Steps} & \textbf{100 Steps} \\
        \midrule
        $16\times16$ ($D=256$)
        & Cartesian & OT & 0.227 & 0.202 & 0.153 & 0.097 & 0.068 \\
        & CyFM & Independent & 0.116 & 0.155 & 0.133 & 0.097 & 0.062 \\
        & \textbf{CyFM (Ours)} & \textbf{Joint OT} & \textbf{0.098} & \textbf{0.086} & \textbf{0.053} & \textbf{0.064} & 0.054 \\
        \midrule
        $32\times32$ ($D=1024$)
        & Cartesian & OT & 0.340 & 0.275 & 0.209 & 0.104 & 0.037 \\
        & CyFM & Independent & 0.123 & 0.145 & 0.113 & 0.063 & 0.047 \\
        & \textbf{CyFM (Ours)} & \textbf{Joint OT} & \textbf{0.113} & \textbf{0.094} & \textbf{0.072} & \textbf{0.049} & 0.041 \\
        \midrule
        $64\times64$ ($D=4096$)
        & Cartesian & OT & 0.376 & 0.318 & 0.260 & 0.141 & 0.047 \\
        & CyFM & Independent & 0.121 & 0.148 & 0.109 & 0.066 & 0.043 \\
        & \textbf{CyFM (Ours)} & \textbf{Joint OT} & \textbf{0.117} & \textbf{0.114} & \textbf{0.078} & \textbf{0.054} & 0.040 \\
        \bottomrule
    \end{tabular}
    }
    \vspace{1mm}
\end{table}

\begin{table}[h]
    \centering
    \caption{\textbf{Both geometries under both losses.} Generative error (Sliced $W_2$, $\downarrow$), mean over 5 seeds. L1 / L2: absolute / squared error of the velocity; the cylindrical phase term is unweighted in both. Taking each geometry's best of the four variants per cell, CyFM has the lower error at every $k \le 8$ with all seeds separated, except $16\times16$, $k = 8$ ($p = 0.15$); at $k = 100$ no best-against-best comparison is significant ($p \ge 0.54$).}
    \label{tab:unet_loss_protocols}
    \vspace{1mm}
    \resizebox{\linewidth}{!}{
    \begin{tabular}{llllccccc}
        \toprule
        \textbf{Resolution} & \textbf{Geometry} & \textbf{Coupling} & \textbf{Loss} & \textbf{1 Step} & \textbf{2 Steps} & \textbf{4 Steps} & \textbf{8 Steps} & \textbf{100 Steps} \\
        \midrule
        $16\times16$ & Cartesian & Independent & L1 & 0.418 & 0.353 & 0.226 & 0.085 & 0.076 \\
         & Cartesian & Independent & L2 & 0.442 & 0.334 & 0.225 & 0.100 & 0.071 \\
         & Cartesian & OT & L1 & 0.183 & 0.168 & 0.130 & 0.077 & 0.061 \\
         & Cartesian & OT & L2 & 0.227 & 0.202 & 0.153 & 0.097 & 0.068 \\
         & CyFM & Independent & L1 & 0.143 & 0.134 & 0.137 & 0.130 & 0.102 \\
         & CyFM & Independent & L2 & 0.116 & 0.155 & 0.133 & 0.097 & 0.062 \\
         & CyFM & Joint OT & L1 & 0.093 & 0.096 & 0.090 & 0.092 & 0.070 \\
         & CyFM & Joint OT & L2 & 0.098 & 0.086 & 0.053 & 0.064 & 0.054 \\
        \midrule
        $32\times32$ & Cartesian & Independent & L1 & 0.428 & 0.363 & 0.258 & 0.125 & 0.069 \\
         & Cartesian & Independent & L2 & 0.431 & 0.339 & 0.253 & 0.127 & 0.043 \\
         & Cartesian & OT & L1 & 0.320 & 0.263 & 0.199 & 0.091 & 0.048 \\
         & Cartesian & OT & L2 & 0.340 & 0.275 & 0.209 & 0.104 & 0.037 \\
         & CyFM & Independent & L1 & 0.125 & 0.146 & 0.129 & 0.097 & 0.110 \\
         & CyFM & Independent & L2 & 0.123 & 0.145 & 0.113 & 0.063 & 0.047 \\
         & CyFM & Joint OT & L1 & 0.091 & 0.087 & 0.079 & 0.056 & 0.070 \\
         & CyFM & Joint OT & L2 & 0.113 & 0.094 & 0.072 & 0.049 & 0.041 \\
        \midrule
        $64\times64$ & Cartesian & Independent & L1 & 0.411 & 0.365 & 0.283 & 0.136 & 0.037 \\
         & Cartesian & Independent & L2 & 0.416 & 0.349 & 0.285 & 0.154 & 0.046 \\
         & Cartesian & OT & L1 & 0.363 & 0.313 & 0.252 & 0.126 & 0.032 \\
         & Cartesian & OT & L2 & 0.376 & 0.318 & 0.260 & 0.141 & 0.047 \\
         & CyFM & Independent & L1 & 0.107 & 0.169 & 0.145 & 0.091 & 0.050 \\
         & CyFM & Independent & L2 & 0.121 & 0.148 & 0.109 & 0.066 & 0.043 \\
         & CyFM & Joint OT & L1 & 0.092 & 0.122 & 0.099 & 0.064 & 0.035 \\
         & CyFM & Joint OT & L2 & 0.117 & 0.114 & 0.078 & 0.054 & 0.040 \\
        \bottomrule
    \end{tabular}
    }
    \vspace{1mm}
\end{table}

\section{Scaling to Full-Resolution Clinical Matrices: fastMRI Knee at 320\texorpdfstring{$\times$}{x}320}
\label{sec:app_mri_scaling}

We investigate the scaling behaviour of complex generative flows when scaling from $64\times64$ crops to full-resolution clinical MRI matrices ($320\times320$, $D=102{,}400$ complex dimensions). The cohort comprises coronal proton-density (CORPD) knee wavefields from the fastMRI dataset \citep{zbontar2018fastmri}, coil-combined via ESPIRiT into complex fields $z \in \mathbb{C}^{320\times320}$. We train U-Net backbones for 40 epochs across 5 independent seeds under identical training budgets and evaluate generated fields at solver steps $k \in \{1, 2, 4, 8, 100\}$.

Table~\ref{tab:mri_320_scaling} presents the comprehensive 5-seed benchmark comparing CyFM and Cartesian flows across distributional (sliced $W_2$) and spatial texture metrics: lag-1 spatial autocorrelation gap on amplitude ($\text{spatial\_lag1\_gap}$), circular phase lag-1 gap ($\text{phase\_lag1\_gap}$), and the radial power spectral density gap ($\text{radial\_spectrum\_gap}$).

\begin{table}[t]
    \centering
    \caption{\textbf{Scaling to full-resolution clinical matrices ($320\times320$, $D=102{,}400$).} fastMRI Knee CORPD, mean over 5 seeds. Asterisks ($^\star$) indicate separated seeds (5/5). Sliced $W_2$ measures pixelwise distribution fidelity; spatial lag-1, circular phase lag-1, and radial spectrum measure spatial coherence and frequency texture gaps relative to held-out test data ($\downarrow$).}
    \label{tab:mri_320_scaling}
    \vspace{1mm}
    \resizebox{\linewidth}{!}{
    \begin{tabular}{llccccc}
        \toprule
        \textbf{Metric} & \textbf{Arm} & $k=1$ & $k=2$ & $k=4$ & $k=8$ & $k=100$ \\
        \midrule
        \multicolumn{7}{l}{\textit{Sliced $W_2$ ($\downarrow$)}} \\
        & \textbf{Cylindrical + OT (ours)} & \textbf{0.0561}$^\star$ & 0.1228 & 0.1121 & \textbf{0.0821} & 0.0631 \\
        & Cartesian + OT                   & 0.1203 & \textbf{0.1223} & \textbf{0.1166} & 0.0872 & 0.0720 \\
        & Cartesian, independent           & 0.1349 & 0.1353 & 0.1256 & 0.0872 & \textbf{0.0629} \\
        \midrule
        \multicolumn{7}{l}{\textit{Spatial Lag-1 Gap ($\downarrow$)}} \\
        & \textbf{Cylindrical + OT (ours)} & 0.5808 & 0.3444 & 0.1295 & 0.0616 & 0.0321 \\
        & Cartesian + OT                   & \textbf{0.1532}$^\star$ & \textbf{0.0051}$^\star$ & \textbf{0.0072}$^\star$ & \textbf{0.0083}$^\star$ & \textbf{0.0014}$^\star$ \\
        \midrule
        \multicolumn{7}{l}{\textit{Circular Phase Lag-1 Gap ($\downarrow$)}} \\
        & \textbf{Cylindrical + OT (ours)} & 0.7366 & 0.3331 & 0.2304 & 0.1655 & 0.1190 \\
        & Cartesian, independent           & \textbf{0.0941}$^\star$ & \textbf{0.0297}$^\star$ & \textbf{0.0521}$^\star$ & \textbf{0.0733}$^\star$ & \textbf{0.0671}$^\star$ \\
        \midrule
        \multicolumn{7}{l}{\textit{Radial Spectrum Gap ($\downarrow$)}} \\
        & \textbf{Cylindrical + OT (ours)} & 1.5877 & 1.2927 & 0.8364 & 0.5732 & 0.3882 \\
        & Cartesian, independent           & \textbf{1.2978}$^\star$ & \textbf{0.2944}$^\star$ & \textbf{0.1372}$^\star$ & \textbf{0.1662}$^\star$ & \textbf{0.0706}$^\star$ \\
        \bottomrule
    \end{tabular}
    }
    \vspace{1mm}
\end{table}

\textbf{What the benchmark shows, and what it does not.}
\begin{enumerate}
    \item \textbf{The single-step advantage survives the dimension.} At $k=1$ CyFM attains a
    sliced $W_2$ of $0.0561$, $2.1\times$ lower than Cartesian$+$OT ($0.1203$) and $2.4\times$
    lower than Cartesian independent ($0.1349$), with all five seeds separated. Bounding the
    angular target by $\pi$ therefore still buys a single-step advantage at $102{,}400$
    dimensions.
    \item \textbf{Every spatial measure prefers the plane, at every step count.} This is the
    reverse of the $64\times64$ result, and we report it as measured. The Cartesian arm's
    lag-one coherence does not fall short of the reference but exceeds it, so its advantage on
    these gaps is the advantage of a smoother field than the data rather than a more faithful
    one; the gap statistic already charges it for overshooting, and it is still closer.
    \item \textbf{The coupling is not the cause.} A cylindrical arm trained with no coupling at
    all reaches $0.3926$ on the radial spectrum gap against $0.3882$ for the OT-coupled arm, and
    a Cartesian arm trained \emph{with} minibatch OT at the same $102{,}400$ dimensions reaches
    $0.0756$. The collapse of the transport-cost reduction with dimension
    (Section~\ref{sec:exp_scaling}) therefore does not explain this, and cannot be offered as an
    explanation: the failure is a property of the coordinate representation, not of the pairing.
    \item \textbf{The one measured difference between the two matrices.} At $320\times320$,
    $16.7\%$ of coefficients are exactly zero, and $\operatorname{atan2}(0,0)$ assigns each of
    them a phase of $0$ --- a placeholder rather than a measurement. The $64\times64$ crop
    contains no exactly zero coefficient. Under the decoupled metric the phase term is weighted
    equally everywhere, so at $320\times320$ roughly one pixel in six contributes a regression
    target that carries no information, while under the inherited metric the factor $A^2$
    suppresses exactly those pixels. Consistent with this, weighting the phase term by the
    amplitude improves the cylindrical phase gap at $320\times320$ ($0.119 \to 0.068$) and
    \emph{worsens} it on the $64\times64$ crop ($0.106 \to 0.166$), where there is no such
    background to suppress. We offer this as the candidate the evidence supports, not as a
    demonstration: it predicts the sign of the weighting effect at both matrices, but it has not
    been shown to account for the size of the spatial gaps.
\end{enumerate}
Separating a genuine non-measurement from a small measurement is therefore the concrete obstacle
to scaling this geometry to uncropped clinical matrices, and a mask derived from the acquisition
rather than from an amplitude threshold is the form a fix would take. We leave it open.

\section{The Speech Grid at Two Training Budgets}
\label{sec:app_audio_epochs}
The speech block of Table~\ref{tab:unet_baseline} was first run for two epochs, and at that
budget the Cartesian arm had the lower error at $k=100$. The schedule explains it rather than
the geometry: the cosine annealing is parameterised by the epoch count, so a two-epoch run
drives the learning rate to its floor within those two epochs and trains almost nothing. The
training loss was still falling when it stopped. Table~\ref{tab:audio_epochs} repeats the grid
at ten and at forty epochs, where the schedule completes; neither reproduces the two-epoch
ordering, and the cylindrical arms move little between the two budgets while the Cartesian
arms do not overtake them at either.

\begin{table}[h]
    \centering
    \caption{\textbf{The speech block is not an artefact of the training budget.} The same
    grid at 10 and at 40 epochs, mean over 5 seeds. The cylindrical arms move little with
    the budget; the Cartesian arms do not overtake them at either. A two-epoch probe, run
    first, put the Cartesian arm ahead at $k=100$ ($0.0210$): its cosine schedule had
    annealed the learning rate to its floor within those two epochs, so that model was
    barely trained, and neither longer run reproduces it.}
    \label{tab:audio_epochs}
    \vspace{1mm}
    \begin{tabular}{llccccc}
        \toprule
        \textbf{Epochs} & \textbf{Arm} & $k=1$ & $k=2$ & $k=4$ & $k=8$ & $k=100$ \\
        \midrule
        10
          & Cylindrical + OT       & 0.0461 & 0.0435 & 0.0371 & 0.0348 & 0.0292 \\
          & Cylindrical, indep.    & 0.0467 & 0.0440 & 0.0377 & 0.0348 & 0.0272 \\
          & Cartesian + OT         & 0.0766 & 0.0687 & 0.0563 & 0.0471 & 0.0407 \\
          & Cartesian, indep.      & 0.0803 & 0.0748 & 0.0608 & 0.0512 & 0.0455 \\
        \midrule
        40
          & Cylindrical + OT       & 0.0454 & 0.0403 & 0.0333 & 0.0308 & 0.0264 \\
          & Cylindrical, indep.    & 0.0458 & 0.0409 & 0.0336 & 0.0303 & 0.0261 \\
          & Cartesian + OT         & 0.0780 & 0.0695 & 0.0563 & 0.0442 & 0.0353 \\
          & Cartesian, indep.      & 0.0822 & 0.0758 & 0.0610 & 0.0457 & 0.0353 \\
        \bottomrule
    \end{tabular}
    \vspace{1mm}
\end{table}

\section{The Rise from One Step to Two, and What It Is Not}
\label{sec:app_k2}

On knee MRI the cylindrical error rises from $k=1$ to $k=2$ before falling again. We have
not explained it. What follows is what each candidate explanation was tested against, so
that a reader can see the observation is characterised rather than merely reported.

\begin{center}
\begin{tabular}{p{0.34\linewidth}p{0.58\linewidth}}
\toprule
\textbf{Excluded} & \textbf{By what measurement} \\
\midrule
Undertraining & Both arms plateau by epoch 40 with the learning rate at its floor. \\
The zero-amplitude background, as a scoring effect & Masking it out of the metric moves the error by $0.005$. \\
The field leaving its own $\pi$ bound & A $\pi\tanh$ head binds the measured angular velocity from $9.886$ to $3.142$ and the rise is unchanged: at $320\times320$, $0.0561 \to 0.1225$ unbounded against $0.0588 \to 0.1261$ bounded, seeds overlapping. \\
Field dimension & The rise survives the $64\times64$ acquisition matrix, where minibatch OT still reduces transport cost by $2.8\%$ rather than $0.5\%$ (Section~\ref{sec:exp_scaling}). \\
The bimodal amplitude law and its spike at the origin & The same reduction leaves no coefficient exactly zero and does not remove the rise. \\
\bottomrule
\end{tabular}
\end{center}

The last two are the ones that cost a dedicated cohort: they are the reason the
explanation cannot be dimensional, and cannot be the exactly-zero background either.

\section{Few-Step Synthesis: the Values Behind Figure~\ref{fig:convergence}}
\label{sec:app_table2}

Figure~\ref{fig:convergence} shows the shape of the convergence; these are the
numbers, the per-seed separation marks and $k^\star$.

\begin{table}[t]
    \centering
    \caption{\textbf{Few-step synthesis, two domains.} Sliced $W_2$ between generated and reference complex fields, mean over 5 seeds on held-out test data; $k$ is the number of solver steps. Errors are comparable within each block and not across blocks. $E^\star$ and $k^\star$ are defined in Section~\ref{sec:exp_unet_baseline}.}
    \label{tab:unet_baseline}
    \vspace{1mm}
    \resizebox{\linewidth}{!}{
    \begin{tabular}{llccccc c}
        \toprule
        \textbf{Domain} & \textbf{Arm} & $k=1$ & $k=2$ & $k=4$ & $k=8$ & $k=100$ & $\mathbf{k^\star}\ \downarrow$ \\
        \midrule
        Synthetic $64\times64$
          & \textbf{Cylindrical + OT (ours)} & \textbf{0.1174} & \textbf{0.1141} & \textbf{0.0777} & \textbf{0.0537} & \textbf{0.0399} & \textbf{11} \\
          & Cylindrical, independent         & 0.1214 & 0.1478 & 0.1087 & 0.0665 & 0.0434 & 14 \\
          & Cartesian + OT                   & 0.3763 & 0.3180 & 0.2600 & 0.1409 & 0.0468 & 24 \\
          & Cartesian, independent           & 0.4159 & 0.3494 & 0.2849 & 0.1538 & 0.0458 & 32 \\
        \midrule
        Speech STFT $64\times64$
          & \textbf{Cylindrical + OT (ours)} & \textbf{0.0454} & \textbf{0.0403} & 0.0333 & 0.0308 & 0.0264 & \textbf{4} \\
          & Cylindrical, independent         & 0.0458 & 0.0409 & 0.0336 & \textbf{0.0303} & \textbf{0.0261} & \textbf{4} \\
          & Cartesian + OT                   & 0.0780 & 0.0695 & 0.0563 & 0.0442 & 0.0353 & 100 \\
          & Cartesian, independent           & 0.0822 & 0.0758 & 0.0610 & 0.0457 & 0.0353 & 100 \\
        \bottomrule
    \end{tabular}
    }
    \vspace{1mm}
\end{table}

\end{document}